\documentclass[runningheads]{llncs}

\usepackage{hyperref}
\usepackage{graphicx}
\usepackage{amsmath,mathtools}
\usepackage{amssymb}
\usepackage{xspace}
\usepackage{stmaryrd}
\usepackage{multirow}
\usepackage{array}
\usepackage{verbatim}

\newcommand{\Terms}{\ensuremath{\mathbf{T}}\xspace}

\newcommand{\vx}{{\ensuremath{\vec{x}}}\xspace}

\newcommand{\vz}{\ensuremath{\vec{z}}\xspace}

\renewcommand{\rule}{\ensuremath{\rho}\xspace}
\newcommand{\auxRule}{\ensuremath{\upsilon}\xspace}

\newcommand{\body}{\ensuremath{\beta}\xspace}
\newcommand{\head}{\ensuremath{\eta}\xspace}

\newcommand{\Instance}{\ensuremath{\mathcal{I}}\xspace}
\newcommand{\I}{\Instance}

\newcommand{\substitution}{\ensuremath{\sigma}\xspace}

\newcommand{\subs}{\substitution}

\newcommand{\Concepts}{\ensuremath{\mathbf{C}}\xspace}
\newcommand{\Individuals}{\ensuremath{\mathbf{I}}\xspace}
\newcommand{\Roles}{\ensuremath{\mathbf{R}}\xspace}
\newcommand{\Variables}{\ensuremath{\mathbf{V}}\xspace}
\newcommand{\Expressions}{\ensuremath{\mathbf{E}}\xspace}

\newcommand{\Set}{\ensuremath{\mathcal{S}}\xspace}

\newcommand\cupEq{\protect{~\cup{\kern -0.5em}=~}}
\newcommand\ncupEq{\protect{~\#{\kern -0.3em}\cup{\kern -0.5em}=~}}
\newcommand{\ELp}{$\mathcal{E \kern -0.2em L}^+$} 
\newcommand{\ELpp}{$\mathcal{E \kern -0.2em L}^{++}$}

\newcommand{\Protege}{Prot\'{e}g\'{e}\xspace}
\newcommand{\Self}{\textsf{Self}\xspace}

\title{Rule-based OWL Modeling \\ with ROWLTab \Protege Plugin} 
\titlerunning{Rule-based OWL Modeling with ROWLTab \Protege Plugin}

\author{Md. Kamruzzaman Sarker\inst{1} \and Adila Krisnadhi\inst{1,2} \and \\ David Carral\inst{3} \and Pascal Hitzler\inst{1}}
\institute{Data Semantics (DaSe) Laboratory, Wright State University, OH, USA \and Faculty of Computer Science, Universitas Indonesia, Depok, Indonesia \and Center for Advancing Electronics Dresden (cfaed), TU Dresden, Germany}

\authorrunning{Sarker, Krisnadhi, Carral, Hitzler}

\begin{document}

\maketitle

\begin{abstract}
  It has been argued that it is much easier to convey logical statements using rules rather than OWL (or description logic (DL)) axioms.  Based on recent theoretical developments on transformations
  between rules and DLs, we have developed ROWLTab, a \Protege plugin that allows users to enter OWL axioms by way of rules; the plugin then automatically converts these rules into OWL 2
  DL axioms if possible, and prompts the user in case such a conversion is not possible without weakening the semantics of the rule. In this paper, we present ROWLTab, together with a user evaluation
  of its effectiveness compared to entering axioms using the standard \Protege interface. Our evaluation shows that modeling with ROWLTab is much quicker than the standard interface, while at the
  same time, also less prone to errors for hard modeling tasks.
\end{abstract}


\section{Introduction}\label{sec:intro}

About a decade ago, not long after description logics \cite{BCMNP07} had been chosen as the basis for the then-forthcoming W3C Recommmendation for the Web Ontology Language OWL \cite{owl2-primer}, a
rather agressively voiced discussion as to whether a rule-based paradigm might have been a better choice emerged 
in the Semantic Web community \cite{HPPH05,PaHo06a,2005-w3c-rules-workshop}. On the one hand, this eventually led to a new W3C Recommendation on the Rule Interchange Format RIF \cite{RIF-Overview},
based on the rules paradigm, while an alternative approach which layered rules on top of the existing OWL standard, known as SWRL \cite{SWRL}, remained a mere W3C member submission. However, SWRL has
proven significantly more popular than RIF. To see this, it may suffice to compare the Google Scholar citation numbers for SWRL -- over 2500 since 2004 -- and RIF -- just over 50 since 2009.

At the same time, researchers kept investigating more elaborate ways to bridge between the two paradigms \cite{KnorrHM12,Krotzsch10,KPSV:Expr-OWL-PL,DBLP:conf/semweb/KrotzschRH08,rkhsv-rr07}, and in
particular how to convert rules into OWL \cite{DBLP:conf/esws/MartinezH12,KrisnadhiMH11,KrotzschMKH11,KrotzschRH08}. These results regarding conversion now make it possible to express axioms first as rules, and only then
to convert them into OWL. We have consistently used this approach to model complex OWL axioms throughout the last few years, as rules are, arguably, easier to understand and produce than OWL
(or description logic) axioms in whichever syntax.

Consider the sentence: ``If a person has a parent who is female, then this parent is a mother'', which we consider to be of medium difficulty in terms of modeling it in OWL. As a first-order
logic rule, this can be expressed as
$$\text{Person}(x) \wedge \text{hasParent}(x,y) \wedge \text{Female}(y) \to \text{Mother}(y).$$
This can be expressed in OWL 
using description logic syntax as follows: 
\begin{gather*}
\text{Female} \sqcap \exists\text{hasParent}^-.\text{Person} \sqsubseteq \text{Mother} \label{eq:dlsyntax}
\end{gather*}

Based on anecdotal evidence, many people find it easier to come up with the rule than directly with the OWL axiom. Following this lead, we have produced a \Protege \cite{protege2015} plugin, called
ROWLTab, which accepts rules as input, and adds them as OWL axioms to a given ontology, provided the rule is expressible by an equivalent set of such axioms. In case the rule is not readily
convertible, the user is prompted and asked whether the rule shall be saved as SWRL rule.

In order to assess the usefulness of the ROWLTab, we have furthermore conducted a user experiment in which we compare the ROWLTab interface for adding axioms to the standard \Protege interface. Our
hypotheses for the user evaluation were that given complex relationships expressed as natural language sentences as above, users will be quicker to add them to an ontology using the ROWLTab than with
the standard \Protege interface, and that they will also make less mistakes in doing so. The first hypothesis has been fully confirmed by our experiment, the second has been partially confirmed.

The rest of the paper will be structured as follows. 
In Section \ref{sec:transform} we explain in more
detail the rule-to-OWL conversion algorithm used. In Section \ref{sec:system} we present the ROWLTab \Protege plugin. In Section \ref{sec:eval} we present our user evaluation and results, and in
Section \ref{sec:conclusion} we conclude.
More information about the plugin can be found at {\small\url{http://daselab.org/content/modeling-owl-rules}}. A preliminary report on the plugin, without evaluation and with much fewer details, was
presented as a software demonstration at the ISWC2016 conference \cite{ROWL-iswc16}.


\section{SWRL Rules to OWL Axioms Transformation}
\label{sec:transform}
\label{section:preliminaries}

In this section we introduce theoretical notions employed across the paper.
Note that, due to space constraints, some of the definitions below are simplified and may not exactly correspond with existing definitions from different sources.

Let \Concepts, \Roles, \Individuals and \Variables be pairwise disjoint, countably infinite sets of \emph{classes}, \emph{properties}, \emph{individuals} and \emph{variables}, respectively, where
$\top, \bot \in \Concepts$, the \emph{universal property} $U \in \Roles$ (i.e., \texttt{owl:topObjectProperty}) and, for every $R \in \Roles$, $R^- \in \Roles$ and $R^{--} = R$.  A \emph{class expression}
is an element of the grammar $\Expressions ::= (\Expressions \sqcap \Expressions) \mid \exists R.\Expressions \mid \exists R.\Self \mid C \mid \{a\}$ where $C \in \Concepts$, $R \in \Roles$ and
$a \in \Individuals$.  Furthermore, let $\Terms = \Individuals \cup \Variables$ be the set of \emph{terms}.  An \emph{atom} is a formula of the form $C(t)$ or $R(t, u)$ where $C \in \Expressions$,
$R \in \Roles$ and $t, u \in \Terms$.
For the remainder of the paper, we identify pairs of atoms of the form $R(t, u)$ and $R^-(u, t)$.  Furthermore, we identify a conjunction of formulas with the set containing all the formulas in the
conjunction and vice-versa.

An \emph{axiom} is a formula of the form $C \sqsubseteq D$ or $R_1 \circ \ldots \circ R_n \sqsubseteq R$ with $C, D \in \Expressions$ and $R_{(i)} \in \Roles$.  A \emph{rule} is a first-order logic
formula of the form $\forall \vx(\body(\vx) \to \head(\vz))$ with \body and \head are conjunctions of atoms, \vx and \vz are non-empty sets of terms where $\vz \subseteq \vx$.  As customary, we
often omit the universal quantifier from rules. 
Axioms and rules are also referred to as \emph{logical formulas}.
Axioms as defined above essentially correspond to OWL 2 EL axioms \cite{OWLtractable} plus inverse property expression, while rules correspond to SWRL rules minus (in)equality and
built-in atoms.

Consider some terms $t$ and $u$ and a conjunction of atoms \body.  We say that \emph{$t$ and $u$ are directly connected in \body} if both terms occur in the same atom in \body.  We say \emph{$t$ and
  $u$ are connected in \body} if there is some sequence of terms $t_1, \ldots, t_n$ with $t_1 = t$, $t_n = u$, and $t_{i-1}$ and $t_i$ are directly connected in \body for every $i = 2, \ldots, n$.

The notions of \emph{interpretation} and of an interpretation \emph{entailing an axiom} follow the standard definitions for description logics \cite{FOST}. For rules \rule of the form $\body \to \head$ we say that an interpretation \emph{\I entails \rule} if, for every substitution \subs we have that $\I, \subs \models \body$ implies $\I, \subs \models \head$, i.e., the semantics of rules follows the standard semantics of first-order predicate logic.
%
%
%
We say that two sets of logical formulas $\Set$ and $\Set'$ are \emph{equivalent} if and only if every interpretation \I that entails \Set also entails $\Set'$ and vice-versa.  Furthermore, we say
that $\Set'$ is a \emph{conservative extension} of \Set if and only if (i) every interpretation that entails $\Set'$ also entails $\Set$ and (ii) every interpretation that entails \Set and is only
defined for the symbols in \Set can be extended to an interpretation entailing $\Set'$ by adding suitable interpretations for additional signature symbols.
It is well-known that a set of logical formulas can be replaced by another set without affecting the outcome of reasoning tasks if the latter set is a conservative extension of the former.



We now formally discuss the transformation of rules into axioms.  We do not include a comprehensive description of this transformation, which was introduced in \cite{KrotzschRH08}, but only
present a simplified version in an attempt to make this publication more self-contained. Specifically, our presentation makes the following assumptions about rules, all without loss of generality.
\begin{enumerate}
\item Rules do not contain constants.  Note that, an atom of the form $R(a, b)$ (resp. $A(a)$) with $a, b \in \Individuals$ in the body of a rule may be replaced by an equivalent atom
  $\exists U.(\{a\} \sqcap \exists R.\{b\})(x)$ (resp. $\exists U.(\{a\} \sqcap A)(x)$) where $x$ is any arbitrarily chosen variable occurring in the body of the rule.  Furthermore, atoms of the
  form $R(x, a)$ with $x \in \Variables$ and $a \in \Individuals$ occurring in the body may be replaced by $\exists R.\{a\}(x)$.  Similar transformations may be applied to the
  head of a rule in order to remove all occurrences of constants.
\item The head of a rule is of the form $C(x)$ or $S(x, y)$ with $C \in \Expressions$, $S \in \Roles$ and $x, y \in \Variables$.
Note that, a rule of the form $\body \to \head$ is equivalent to a set of rules $\{\body \to \head_1, \ldots, \body \to \head_n\}$ provided that $\head = \head_1 \cup \ldots \cup \head_n$.
\item All of the variables in the body of a rule are connected.  If two variables $x$ and $y$ are not connected in the body of a rule, we may simply add the atom $U(x, y)$ to the body of the rule resulting in a semantically equivalent rule.
\end{enumerate}

The preprocessing implied by the aforementioned assumptions has been implemented in the ROWLTab plugin and thus, such constraints need not be considered by the end users.
Moreover, such preprocessing is carefully implemented in an attempt to minimize the number of necessary modifications for a rule to satisfy assumptions (1-3).
We  now proceed with the definition of our translation.
\begin{definition}
\label{definition:transformation1}
Given some rule $\rule = \body \to \head$, let $\delta(\body \to \head)$ be the rule that results from exhaustively applying transformations (1-3) where (1) and (2) should be applied with higher priority than (3).
\begin{enumerate}
\item Replace every atom of the form $R(x, x)$ in \body or in \head with $\exists R.\Self(x)$.
\item Replace every maximal subset of the form $\{C_1(y) \ldots, C_n(y)\} \subseteq \body$ with the atom $C_1 \sqcap \ldots \sqcap C_n(y)$.
\item For every variable $y$ not occurring in \head that occurs in exactly one binary atom in \body of the form $R(z, y)$, do the following:
\begin{itemize}
\item If there is some atom of the form $C(y) \in \body$, then replace the atoms $R(z, y)$ and $C(y)$ in \body with the atom $\exists R.C(z)$.
\item Otherwise, replace the atom $R(z, y)$ in \body with $\exists R.\top(z)$.
\end{itemize}
\end{enumerate}
\end{definition}

\begin{example}
\label{example1}
Consider the rule $\rho = \text{Person}(x)  \wedge \text{hasParent}(x, y) \wedge \text{Female}(y) \to \text{Mother}(y)$.
Then, the transformation presented in the previous definition would sequentially produce the following sequence of rules
\begin{align*}
(\exists \text{hasParent}^-.\text{Person})(y) \wedge \text{Female}(y) &\to \text{Mother}(y) \\
(\exists \text{hasParent}^-.\text{Person} \sqcap \text{Female})(y) &\to \text{Mother}(y)
\end{align*}
\end{example}

Rule $\delta(\rho)$ from the previous example can be directly transformed into an axiom as indicated in the following lemma.

\begin{lemma}
\label{lemma1}
Consider some rule $\rule$.
If $\delta(\rule)$ is of the form $C(x) \to D(x)$, then \rule is equivalent to the axiom $C \sqsubseteq D$.
\end{lemma}
\begin{proof}
Let \auxRule and $\auxRule'$ be some rules such that $\auxRule'$ results by applying some of the transformations (1-3) introduced in Definition 1 to \auxRule.
Note that, by definition, we can conclude equivalency between \auxRule and $\auxRule'$.
Thus, we can show via induction that $\rule$ is equivalent to $\delta(\rule)$.
Furthermore, if $\delta(\body \to \rule)$ is of the form $C(x) \to D(x)$, then, by the definition of the semantics of rules and axioms, $C \sqsubseteq D$ is equivalent to $\delta(\body \to \rule)$.
Since the equivalence relation is transitive, we can conclude that \rule is equivalent to $C \sqsubseteq D$.
\end{proof}

As indicated by Lemma \ref{lemma1}, rule \rule from Example \ref{example1} is equivalent to the axiom
$\exists \text{hasParent}^-.\text{Person} \sqcap \text{Female} \sqsubseteq \text{Mother}$.

\begin{lemma}
Consider some rule $\rule$.
If the rule $\delta(\rule)$ is of the form $\bigwedge_{i = 2}^n (C_i(x_{i-1})  \wedge R_i(x_{i-1}, x_i)) \wedge C_n(x_n) \to S(x_1, x_n)$, then the set of axioms
$\{C_i \sqsubseteq \exists R_{C_i}.\Self \mid i = 1, \ldots, n\} \cup \{R_{C_1} \circ R_1 \circ \ldots \circ R_{C_{n-1}} \circ R_n \circ R_{C_n} \sqsubseteq S\}$
where all $R_{C_i}$ are fresh properties unique for every class $C_i$ is a conservative extension of the rule \rule.
\end{lemma}
\begin{proof}
As shown in proof of Lemma \ref{lemma1}, rules \rule and $\delta(\rule)$ are indeed equivalent.
Thus, the lemma follows from the fact that the set of rules presented in the statement of the lemma is a conservative extension of $\delta(\rule)$.
\end{proof}

We finalize the section with some brief comments about the presented transformation.
First, the transformation is sound, but has not been proven to be complete.
That is, there may be some rules which are actually expressible as axioms, but cannot be handled by our translation algorithm.
Moreover, the transformation of a given rule may in some cases produce axioms that violate some of the global syntactic restrictions of OWL 2 DL, such as regularity restriction on property inclusion/hierarchy, when added to an ontology.

\section{Plugin Description and Features}\label{sec:system}





\begin{figure}[tp]
  \centering
  \includegraphics[width=\textwidth]{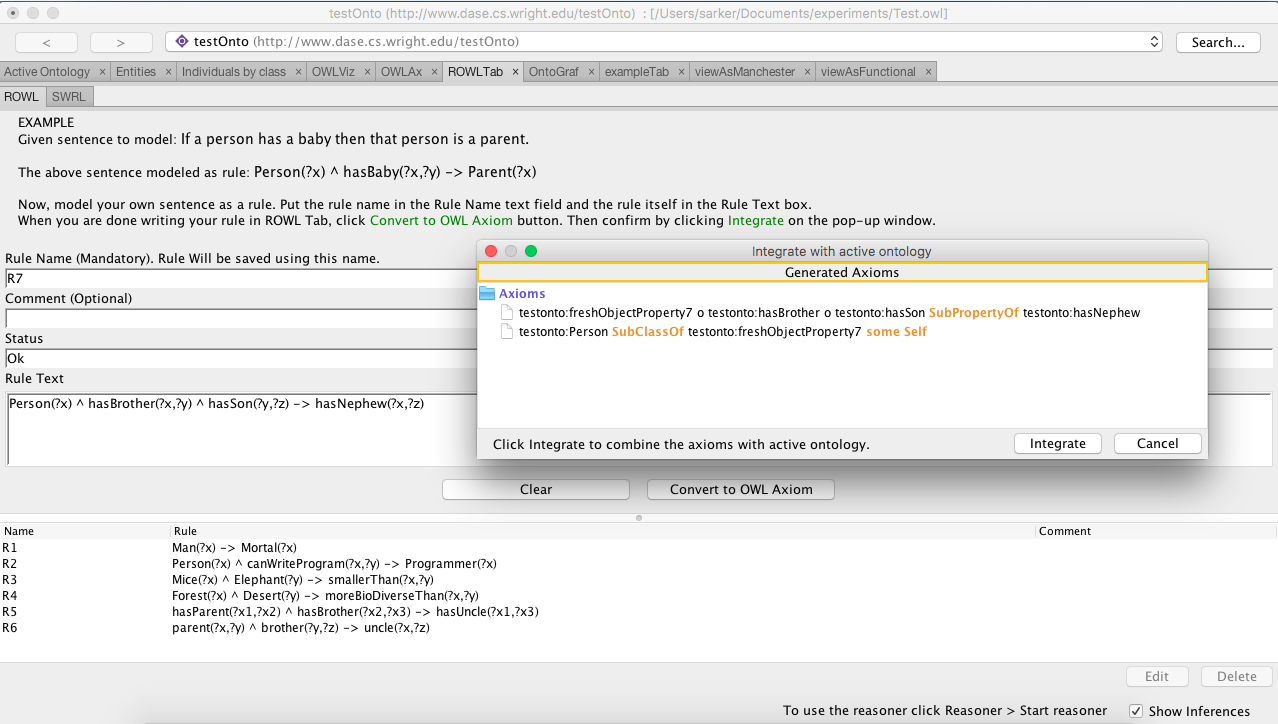}
  \caption{The ROWLTab interface with generated axioms. }
  \label{fig:rowl}
\end{figure}

Figure \ref{fig:rowl}  shows the user interface of the ROWLTab plugin with generated axioms from a rule and also  
shows previously saved rules in the bottom part of the user interface. 
As seen in the figure, the plugin consists of two tabs: ROWL and SWRL. The latter is really the SWRLTab input interface, while the former is our implementation of rule-to-OWL conversion
functionality. We have in fact reused the source code of the SWRLTab, kept its functionality intact and added extra functionality by means of the ROWL tab.
The upper part of the interface is for rule insertion and the bottom part is for rule modification.  At the top a modeling example is also shown. 

A user can enter a rule in the ``Rule Text'' input box using the standard SWRL syntax, e.g. 
\begin{center}
\verb|Person(?x) ^ hasChild(?x,?y) ^ Female(?y) -> hasDaughter(?x,?y)|
\end{center}
Every rule needs a distinct name to be eligible to be saved for later modification. A suggested rule name is automatically generated but the user also has the option to change the rule name; this can
be done in the ``Rule Name'' input box. The user can also give annotations to a rule in the ``Comment'' input box. The plugin does syntax checking of the rule, but nothing more sophisticated,
e.g. tautologies can also be entered, and checks for global constraints like RBox regularity, which are required for the ontology to stay within OWL 2 DL, are not performed. Since rule-to-OWL
conversion often results in the use of property chains, extra care is needed by the modeler to ensure compliance, if compliance is desired; this can also be checked by using a reasoner such as HermiT
\cite{msh09hypertableau} from within \Protege. When the user is writing a rule it checks whether a predicate is already declared or not. If not declared it will show that the predicate is invalid and
the user needs to declare it before the rule can be converted to OWL.  Auto completion of predicate names is also supported. The user can use tab-key for auto-completion to existing class, object
properties, data properties and individual names.

Figure \ref{fig:rowl} also shows a button ``Convert to OWL Axiom'' below the ``Rule Text'' input box. This button is initially deactivated and if the inserted rule is syntactically correct then this
button becomes clickable.  When the ``Convert to OWL Axiom'' button is clicked, ROWLTab will attempt to apply the rule-to-OWL transformation described in the previous section to the given rule. If
successful, a pop-up will appear displaying one or more OWL axioms resulting from the transformation, presented in Manchester syntax. These axioms can then be integrated into the active ontology by
clicking the ``Integrate'' button of the pop-up interface.
If the given rule cannot be transformed into OWL axioms, ROWLTab will prompt the user if (s)he still want to insert the rule into the ontology as a SWRL rule. If the user agrees, ROWLTab will switch
to its SWRL tab and proceed in the same way as adding a rule via the original
SWRLTab. 

As described in Section \ref{sec:transform}, the translation of a rule into OWL axioms may sometimes require the introduction of fresh object properties, which will be automatically created by the
plugin when necessary.  The namespace for these fresh object properties is taken from the default namespace. To create a unique fresh object property, the plugin counts the number of existing fresh
object properties in the active ontology (including imports) then increments the counter by 1 and creates the new object property with the incremented counter as part of its identifier.

Once the axioms generated from a rule are added to the ontology, the rule will be saved with the ontology for later modification; the rule is in fact added as an annotation to every OWL axiom
generated by the rule. Figure~\ref{fig:rowl} shows saved rules displayed on the bottom left of the ROWLTab plugin. A user can modify or delete rules at any time.  If a rule is modified or deleted, the
axioms generated by that rule will be affected. That means if a rule is deleted then the axioms generated by the rule will also be deleted. To edit a rule which was previously used to generate axioms
the user needs to select the rule first and then click the ``Edit'' button at the bottom right part of the interface. The rule will then appear in the ``Rule Text'' input box for modification. The
user can also double click on the rule to edit that rule. To delete a rule, the user needs to select that rule and then click the ``Delete'' button on the bottom right of the interface.

A feature of ROWLTab not found in the SWRLTab is the possibility to automatically add declarations for classes and properties if the inserted rule contains classes or properties not yet defined in the
ontology.  For example, in the rule above, the original SWRLTab requires that \textsf{hasChild} and \textsf{hasDaughter} be already defined as object properties, and \textsf{Person} and
\textsf{Female} as classes in the ontology.  This means that the user 
does not need to first exit the plugin and declare the classes and properties outside the ROWLTab.

Another feature of the ROWLTab plugin is that it actually works as a superset of SWRLTab plugin. So if a user need to work with the SWRLTab plugin, the user does not need to install the SWRLTab plugin
separately, as the ROWLTab contains a full instance of the SWRLTab plugin. If the ROWLTab plugin is installed the user only need to switch the tab from ROWL to SWRL to get the full SWRLTab
functionality. This also creates a limitation that when a new version of SWRLTab is available the developer of ROWLTab has to embed the newer version of SWRLTab explicitly. 

To manage the source code and to be able to modify the source code efficiently in the future we have separated the view module from the control module. The view module consist of the user interface
and the control module implements the rule-to-OWL transformation. Besides those two modules we have a separate listener module, which acts as a bridge between the view and controller module. We have
used Maven as our build system to easily manage the dependency of various APIs. This plugin is open source and the source code is available at the DaseLab website ({\small
  \url{http://dase.cs.wright.edu/content/modeling-owl-rules}}).



\section{Evaluation}\label{sec:eval}

For the evaluation of the ROWLTab plugin, we conducted a user evaluation to answer the following three questions:
\begin{itemize}
\item Is writing OWL axioms into \Protege via the ROWLTab plugin quicker than writing them directly through the standard \Protege interface?
\item Is writing OWL axioms into \Protege via ROWLTab plugin less error-prone than writing them directly through the standard \Protege interface?
\item Do users view modeling OWL axioms via ROWLTab plugin to be an easier task than directly through the standard \Protege interface?
\end{itemize}

The evaluation was conducted by asking the participants to model a set of natural language sentences as rules using the ROWLTab or as OWL axioms using the standard \Protege interface. We
recorded time and number of keyboard and mouse clicks required for each question (see Section \ref{subsec:time}) and also recorded the responses which we subsequently assessed for correctness (see
Section \ref{subsec:correctness}). Finally, the participants answered a brief questionnaire (see Section \ref{subsec:survey}).

Before describing the experiment in detail, we would like to encourage the readers to do it themselves; it should take less than an hour, the software can be obtained from the ROWLTab website already
indicated.

For the experiment we recruited 12 volunteers from among the graduate students at Wright State University. Our sole selection criterion was that the participants had at least some basic knowledge of
OWL, and had at least minimal exposure to \Protege. All participants were then given a half-hour briefing in which we explained, by means of examples, how to model natural language sentences with and
without the ROWLTab in \Protege.

Each participant was given the same twelve natural language sentences to model. The sentences are listed in Table \ref{tab:evalquest} where group A consists of sentence 1 to 6 and group B consists of
sentence 7 to 12. As indicated in the table, each group contains two easy, two medium, and two hard sentences to model. Each participant modeled one of the sets of sentences using the ROWLTab and the
other group without using the ROWLTab, and we randomly assigned whether the participant will model Group A using the ROWLTab or Group B using the ROWLTab. In order to minimize learning effects which
may come from different sentences, we made sure that for each sentence in Group A there is a very similar sentence in Group B, and vice-versa: Each sentence number $n$ in Group A corresponds to
sentence number $n+6$ in Group B. We furthermore randomized whether the participant will first model using ROWLTab, and then without the ROWLTab, or vice-versa, also to control for a possible
learning effect during the course of the experiment. There was no time limit for the modeling; participants were informed that it should usually take no longer than an hour to model all twelve
sentences. Participants were also informed that they cannot go back to earlier sentences during the course of the experiment.


\begin{table}[t]
\caption{Evaluation Questions}\label{tab:evalquest}
\begin{tabular}{lp{.405\textwidth}|lp{.395\textwidth}|c}
  \multicolumn{2}{c|}{Group A} & \multicolumn{2}{c|}{Group B} & Difficulty\\
  \hline
  1. & Every father is a parent. & 7. & Every parent is a human. & \multirow{3}{*}{easy}\\
  2. & Every university is an educational institution. & 8. & Every educational institution is an organization. & \\
  \hline
  3. & If a person has a mother then that mother is a parent. & 9. & If a person has a parent who is female, then this parent is a mother. & \multirow{4}{*}{medium}\\
  4. & Any educational institution that awards a medical degree is a medical school. & 10. & Any university that is funded by a state government is a public university. & \\
  \hline
  5. & If a person's brother has a son, then that son is the first person's nephew. & 11. & If a person has a female child, then that person would have that female child as her daughter. 
                                 & \multirow{5}{*}{hard}\\
  6. & All forests are more biodiverse than any desert. & 12. & All teenagers are younger than all twens. &
\end{tabular}
\end{table}

Our categorization into easy, medium, and hard sentences was done as follows: Easy sentences expressed simple subclass relationships. Medium sentences required the use of property restrictions to
model them in OWL; the medium sentence 9 was discussed in Section \ref{sec:intro}. Hard sentences could only be expressed using two or three OWL axioms, together with a technique called
\emph{rolification}~\cite{KrisnadhiMH11,KrotzschMKH11}. For example, sentence 5 when expressed as a rule becomes
$$\text{Person}(x) \wedge \text{hasBrother}(x,y) \wedge \text{hasSon}(y,z) \to \text{hasNephew}(x,z).$$
In order to express this sentence as OWL axioms, one first has to \emph{rolify} the class Person by adding the axiom $\text{Person} \sqsubseteq \exists R_\text{Person}.\text{Self}$,
where $R_\text{Person}$ is a fresh property name, and to then add the property chain axiom
$$
R_\text{Person} \circ \text{hasBrother} \circ \text{hasSon} \sqsubseteq \text{hasNephew}.
$$


We informed all the participants regarding the total number of easy, medium, and hard sentences the participants would face.  
With each sentence, we also displayed the suitable class and property-names which had been pre-defined by us, i.e. the participants did not have to declare them in \Protege, and directed participants
to use the displayed class and property names to the maximum extent possible. For example, the pre-defined classes and properties for sentence 9 were Person, hasParent, Female, Mother, while for
sentence 5 they were Person, hasBrother, hasSon, hasNephew. An exception to this is when modeling the hard sentences (5, 6, 11, and 12) via standard \Protege. Those sentences contain class names that
need to be rolified, which necessitates one to declare one or more fresh object properties. In this case, we informed the participants that the hard sentences may require them to declare additional
object properties without disclosing that this is due to rolification.

\subsection{Time Used For Modeling}\label{subsec:time}

Our hypothesis was that, on medium and hard sentences, participants would be able to model quicker with the ROWLTab than without it. Cumulated data is given in Table \ref{tab:time}. 

\begin{table}[t]
\centering
\caption{Average and standard deviation of time (in seconds), number of clicks (keyboard and mouse), and correctness score per difficulty category of sentences.}\label{tab:time}
\setlength{\tabcolsep}{2pt}
\setlength{\extrarowheight}{1pt}

\setlength{\tabcolsep}{7pt}
\begin{tabular}{c|cccccc}
  Sentence  & \multicolumn{2}{c}{Time (in secs)} & \multicolumn{2}{c}{\# clicks} & \multicolumn{2}{c}{Correctness} \\ \cline{2-7}
  Category & \Protege & ROWL & \Protege & ROWL & \Protege & ROWL \\ 
  & avg/std & avg/std & avg/std & avg/std & avg/std & avg/std \\ \hline
  easy     &  \phantom{0}79/\phantom{0}41 & \phantom{0}47/\phantom{0}9 & \phantom{0}44/\phantom{0}38 & \phantom{0}59/\phantom{0}19 & 2.9/0.3 & 2.9/0.3\\
  medium   & 312/181 & 116/61 & 216/131 & 141/\phantom{0}91 & 2.2/0.5 & 2.5/0.8 \\
  hard     & 346/218 & 160/66 & 351/318 & 228/168 & 0.9/0.7 & 2.5/0.7  \\
\end{tabular}
\end{table}

For the statistical analysis, our null hypothesis was that there is no difference between the time taken with ROWLTab versus \Protege. Since each participant had modeled sentences from each difficulty
class, we could perform a paired (two-tailed) t-test -- note that assuming normal distributions appears to be perfectly reasonable for this data. For the medium sentences the null hypothesis was
rejected with $p \approx 0.002<0.01$. For the hard questions the null hypothesis was rejected with $p \approx 0.020<0.05$. Both results are statistically significant with $p<0.05$, thus confirming our
hypotheses.

Interestingly, if we run the same t-test also on the easy sentences, the same null hypothesis is also rejected with $p \approx 0.019 <0.05$. We will reflect on this further below. 

In order to aid us in interpreting the results, we also recorded the number of clicks (keyboard plus mouse) required for modeling each sentence; cumulative data is provided also in Table
\ref{tab:time}. If we run the number of clicks through the same t-test as before, the null hypothesis being that there is no difference between the two interfaces. The corresponding $p$-values are
0.092 (easy), 0.030 (medium) and 0.173 (hard), i.e. we have $p<0.05$ only for the medium sentences.  The click analysis may provide us with a partial answer to the better performance of the ROWLTab
regarding time used: Fewer clicks may in this case simply translate into less time required. However, this observation does not explain the data for easy and hard questions. We will return to this
discussion at the end of Section \ref{subsec:correctness}, after we have looked at answer correctness.



\subsection{Correctness of Modeled Axioms}\label{subsec:correctness}

Our hypothesis was that for medium and hard questions, participants would provide more correct answers with the ROWLTab than without it. To see if we can confirm this hypothesis, we verify the
correctness of the axioms in the OWL files obtained from the participants. The correct set of axioms for each modeling question is given in Table~\ref{tab:answerkeys}. Since the sentences are short
and the resulting OWL axioms are relatively simple, the verification was done manually. Also, for modeling tasks where the participants were asked to model the sentence via ROWLTab, we check the
correctness of the rules by examining the OWL axioms obtained after translation, which are annotated with information regarding the actual rule input given to ROWLTab. We then assign a score of 0, 1,
2, or 3 to each answer from the participants as follows:
\begin{itemize}
\item Modeling a sentence via ROWLTab, the score is:
  \begin{itemize}
  \item 3 if the participant's rule is fully correct (equivalent to the answer key),
  \item 2 if the participant's mistakes are only in the incorrect use of variables (wrong placement, missing/spurious variables), i.e., the rule still employs the correct predicates in the rule body
    and head and no spurious predicates are used,
  \item 1 if there's a missing predicate in the participant's rule or spurious predicates are used that makes the rule not equivalent to the correct answer,
  \item 0 if the participant provides no answer.
  \end{itemize}
\item Modeling a sentence via the standard \Protege interface, the score is:
  \begin{itemize}
  \item 3 if the participant's OWL axioms are fully correct,
  \item 2 if the participant's OWL axioms employ the correct set of class and property names, but there is a mistake in the use of logical constructs,
  \item 1 if there is a missing or spurious class names or property names, or even missing some necessary OWL axioms,
  \item 0 if the participant provides no answer.
  \end{itemize}
\end{itemize}

\begin{table}[t]
  \caption{Answers to Evaluation Questions -- Ru\emph{i} and Ax\emph{i} are answers for question \emph{i} in the form of rule and OWL axioms, resp. where $R_{1},\dots,R_{7}$ are fresh (object)
    properties generated due to rolification}\label{tab:answerkeys}
\setlength{\extrarowheight}{3pt}
\setlength{\tabcolsep}{4pt}
\footnotesize\centering
\begin{tabular}{l}
  \hline
  Ru1: $\text{Father}(x) \rightarrow \text{Parent}(x)$ \quad Ax1:
  $\text{Father} \sqsubseteq \text{Parent}$\\ \hline

  Ru2: $\text{University}(x) \rightarrow \text{EducationalInstitution}(x)$ \\
  Ax2: $\text{University} \sqsubseteq \text{EducationalInstitution}$ \\ \hline

  Ru3: $\text{Person}(x) \wedge \text{hasMother}(x,y) \rightarrow \text{Parent}(y)$ \\
  Ax3: $\exists\text{hasMother}^{-}.\text{Person} \sqsubseteq \text{Parent}$ \\ \hline

  Ru4: $
  \begin{aligned}[t]
    \text{EducationalInstitution}(x) \wedge \text{awards}(x,y) &\wedge \text{MedicalDegree}(y) \\ &\rightarrow \text{MedicalSchool}(x)
  \end{aligned}
$ \\
  Ax4: $\text{EducationalInstitution} \sqcap \exists \text{awards}.\text{MedicalDegree}  \sqsubseteq \text{MedicalSchool}$ \\ \hline

  Ru5: $\text{Person}(x) \wedge \text{hasBrother}(x,y) \wedge \text{hasSon}(y,z) \rightarrow \text{hasNephew}(x,z)$ \\
  Ax5: $\text{Person} \sqsubseteq \exists R_{1}.\Self, \quad \text{R}_{1} \circ \text{hasBrother} \circ \text{hasSon} \sqsubseteq \text{hasNephew}$ \\ \hline

  Ru6: $\text{Forest}(x) \wedge \text{Desert}(y) \rightarrow \text{moreBiodiverseThan}(x,y)$ \\ 
  Ax6: $\text{Forest} \sqsubseteq \exists R_{2}.\Self, \quad \text{Desert} \sqsubseteq \exists R_{3}.\Self, \quad R_{2} \circ U \circ R_{3} \sqsubseteq \text{moreBiodiverseThan}$\\ \hline

  Ru7: $\text{Parent}(x) \rightarrow \text{Human}(x)$ \quad
  Ax7: $\text{Parent} \sqsubseteq \text{Human}$ \\ \hline

  Ru8: $\text{EducationalInstitution}(x) \rightarrow \text{Organization}(x)$ \\
  Ax8: $\text{EducationalInstitution} \sqsubseteq \text{Organization}$ \\ \hline

  Ru9: $\text{Person}(x) \wedge \text{hasParent}(x,y) \wedge \text{Female}(y) \rightarrow \text{Mother}(x)$ \\
  Ax9: $\text{Person} \sqcap \exists \text{hasParent}.\text{Female} \sqsubseteq \text{Mother}$ \\ \hline

  Ru10: $\text{University}(x) \wedge \text{fundedBy}(x,y) \wedge \text{StateGovernment}(y) \rightarrow \text{PublicUniversity}(x)$ \\
  Ax10: $\text{University} \sqcap \exists \text{fundedBy}.\text{StateGovernment} \sqsubseteq \text{PublicUniversity}$ \\ \hline

  Ru11: $\text{Person}(x) \wedge \text{hasChild}(x,y) \wedge \text{Female}(y) \rightarrow \text{hasDaughter}(x,y)$ \\
  Ax11: $ \text{Person} \sqsubseteq \exists R_{4}.\Self, \quad \text{Female} \sqsubseteq \exists R_{5}.\Self, \quad R_{4} \circ \text{hasChild} \circ R_{5} \sqsubseteq \text{hasDaughter}$ \\ \hline

  Ru12: $\text{Teenager}(x) \wedge \text{Twen}(y) \rightarrow \text{youngerThan}(x,y)$ \\
  Ax12: $\text{Teenager} \sqsubseteq \exists R_{6}.\Self, \quad \text{Twen} \sqsubseteq \exists R_{7}.\Self, \quad R_{6} \circ U \circ R_{7} \sqsubseteq \text{youngerThan}$ \\ \hline
\end{tabular}

\end{table}


The average and standard deviation of the correctness score for easy, medium, and hard questions can be found in Table~\ref{tab:time}.  Here, our null hypothesis was that there is no difference in
correctness of the answers given with ROWLTab versus \Protege. For the same reasons as before, we thus performed a paired (two-tailed) t-test, the null hypothesis being that there be no difference
whether using ROWLTab or not. For the medium sentences we obtained $p \approx 0.18>0.05$, so the null hypothesis could not be rejected. But for the hard questions the null hypothesis was rejected
with $p \approx 0.0001<0.01$. The latter result is statistically significant with $p<0.01$. If we run the same t-test also on the easy sentences, the same null hypothesis cannot be rejected; we in
fact obtain $p \approx 1.0000$.

We thus confirm our hypothesis that ROWLTab helps users in modeling hard sentences correctly; however we could not confirm this for medium sentences on this population of participants. It could be
hypothesized that the participants were sufficiently familiar with \Protege to perform well on medium difficulty, thus use of the ROWLTab only had an effect on time used, as shown in Section
\ref{subsec:time}.

At the same time, the correctness analysis also sheds further light on the hard questions: While participants used less time for these on the ROWLTab, they did not use significantly fewer clicks;
however answer correctness was much higher on the ROWLTab. This seems to indicate that the additional time using \Protege was spent thinking (and indeed, rather unsuccessfully) about the problem,
while this additional thinking was not required when using the ROWLTab.


\subsection{Participant Survey}\label{subsec:survey}

We finally used a questionnaire with four questions to assess the subjective value which the use of the ROWLTab had to the participants. For this, we  asked all participants to indicate to what extent they agree with each of the following statements.
\begin{enumerate}
\item ROWLTab is a useful tool to help with ontology modeling.
\item Modeling rules with ROWLTab was easier for me than modeling without it.
\item Given some practice, I think I will find modeling rules with the ROWLTab easier than modeling without it.
\item The ROWLTab is better for ontology modeling than the SWRLTab.
\end{enumerate}

Participants were asked to click, on screen whether they agree with each statement, on a scale from -3 (strongly disagree) to +3 (strongly agree). It turns out that participants agreed highly with all three statements:

\medskip

\centerline{\begin{tabular}{l|c|c}
Question Number & mean & standard deviation\\
\hline
1 (ROWL is a useful tool.) & 2.83 & 0.39\\
2 (ROWL makes modeling easier.) & 3.00 & 0.00\\
3 (Modeling with ROWL easier with some practice.) & 2.75 & 0.45\\
4 (ROWLTab better than SWRLTab) & 1.75 & 1.22
\end{tabular}}

\medskip

In assessing these responses, we need to be aware that the pool of participants came from the investigators' institution, and many of them were either associated with the investigators' lab or had attended classes by one of the investigators. Hence the scores should be interpreted with caution. Nevertheless, the scores for the first three questions indicate strong agreement with the usefulness of the ROWLTab. 

Regarding the fourth question, it should be noted that our briefing did not include a briefing on the SWRLTab. As discussed in Section \ref{sec:system}, the user interaction of the ROWLTab is very similar to that of the SWRLTab, so the only substantial difference would be in the fact that the ROWLTab produces OWL axioms, while the SWRLTab produces SWRL axioms with a different semantics. We do not know to what extent the participants were aware of this difference. A quarter of the participants answered this question with ``0'' (neutral). Results of the experiment can be found at  {\small \url{http://dase.cs.wright.edu/content/rowl} and raw result can be found at {\small
  \url{https://github.com/md-k-sarker/ROWLPluginEvaluation/tree/master/results}} .



\section{Conclusions and Further Work}\label{sec:conclusion}

We have presented the \Protege ROWLTab plugin for rule-based OWL modeling in \Protege, and its underlying algorithms. We have furthermore reported on a user evaluation for assessing the improvements
arising from the use of ROWLTab.

\begin{table}[t]
\begin{center}
\caption{Summary of evaluation results. Entries indicate whether the difference between using the ROWLTab and not using it were statistically significant.}\label{tab:evalsum}
\begin{tabular}{c|c|c|c|}
category & time & clicks & correctness\\
\hline
easy & significant ($p<0.05$) & not significant &  not significant\\
medium & significant ($p<0.01$) & significant ($p<0.05$) & not significant\\
hard & significant ($p<0.05$) & not significant & significant ($p<0.01$)
\end{tabular}
\end{center}
\end{table}

The evaluation results are summarized in Table \ref{tab:evalsum}: We have a significant time improvement in all three categories (it was hypothesized by us only for medium and hard sentences). In the
medium category, where answer correctness was not significantly different, the ROWLTab required significantly less clicks. In the hard category, the difference in answer correctness was also
significant.

The evaluation results are rather encouraging, and we also already received direct feedback from users that the ROWLTab is considered very useful. But while basic functionality is already in place, we
already see further improvements that can be made to the plugin:
\begin{itemize}
\item When rolification is used for the transformation of a rule to OWL, the ROWLTab currently invents an artificial property name for the fresh object property. It may be helpful to more directly
  support a renaming of these properties, or to come up with a standard naming scheme for properties arising out of rolification. Note, however, that it is not sufficient to have one fresh property
  for each defined atomic class, as in some cases complex classes need to be rolified \cite{KrisnadhiMH11}.
\item The translation of rules into OWL often leads to the use of property chains, which may result in a non-regular property hierarchy, thus violating a global syntactic restriction of OWL 2
  DL. While standard tools such as reasoners, which can be called from within \Protege, can detect this issue, it may be helpful to catch this earlier, e.g. directly at the time when a rule is
  translated.
\item Currently, if a rule is input which cannot be translated to OWL, it is simply saved as a SWRL rule, i.e., with a significantly modified (and, in a sense, restricted) semantics. However, through
  the use of so-called \emph{nominal schemas} \cite{KrotzschMKH11} it is possible to recover more of the first-order semantics of the input rules, and it has even been shown that the use of such
  nominal schemas can lead to performance improvements of reasoners compared to SWRL \cite{SteigmillerGL14}.
\end{itemize}

More substantial possible future work would carry the ROWLTab theme beyond the basic rule paradigm currently supported:
\begin{itemize}
\item The rule syntax could be extended to allow for capturing OWL features which cannot be expressed by means of the basic rules currently supported. In particular, these would be right-hand side (head) disjunctions and existentials as well as cardinality restrictions, as well as left-hand side universal quantifiers. It would even be conceivable to add additonal shortcut notation, e.g. for witnessed universals \cite{CarralKRH14}, or for nominal schemas \cite{KrotzschMKH11}.
\item The development of a full-blown rule syntax for all of OWL 2 DL would then also make it possible to perform all ontology modeling using rules, i.e., to establish an interface where the user would get a pure rules view on the ontology, if desired.
\end{itemize}

We are looking forward to feedback by ontology modelers on the route which we should take with the plugin in the future.


\bigskip

\noindent\emph{Acknowledgements.} This work was supported by the National Science Foundation under award 1017225 \emph{III: Small: TROn -- Tractable Reasoning with Ontologies} and the German Research Foundation (DFG) within the Cluster of Excellence ``Center for Advancing Electronics Dresden'' (cfaed).
We would also like to thank Tanvi Banerjee and Derek Doran for some advise on statistics.

\urlstyle{same}
\bibliographystyle{splncs03}
\bibliography{short}



\end{document}